\documentclass[lettersize,journal]{IEEEtran}
\usepackage{graphicx} 
\usepackage{amsmath}
\usepackage{amssymb}
\usepackage{cite}
\usepackage{optidef}
\usepackage{xcolor}
\usepackage{balance}
\usepackage{booktabs}
\usepackage{algorithm}
\usepackage{algpseudocode}
\usepackage{hyperref}
\usepackage{url}
\usepackage{amsthm}

\newtheorem{assumption}{Assumption}
\newtheorem{theorem}{Theorem}
\newtheorem{definition}{Definition}

\newtheorem{remark}{Remark}
\newtheorem{example}{Example}

\begin{document}
\title{Formation-Aware Adaptive Conformalized Perception for Safe Leader–Follower Multi-Robot Systems}

\author{Richie R. Suganda$^{1}$ and Bin Hu$^{1,2}$
\thanks{$^{1}$ Richie R. Suganda and Bin Hu are with the Department of Electrical and Computer Engineering, University of Houston, Houston, TX 77004, USA. {\tt rrsugand@cougarnet.uh.edu; bhu12@uh.edu}.}
\thanks{$^{2}$ Bin Hu is with the Department of Engineering Technology, University of Houston, Houston, TX 77004, USA. {\tt bhu12@uh.edu}.}}


\maketitle

\begin{abstract}
This paper considers the perception safety problem in distributed vision-based leader–follower formations, where each robot uses onboard perception to estimate relative states, track desired setpoints, and keep the leader within its camera field of view (FOV). Safety is challenging due to heteroscedastic perception errors and the coupling between formation maneuvers and visibility constraints. We propose a distributed, formation-aware adaptive conformal prediction method based on Risk-Aware Mondrian CP to produce formation-conditioned uncertainty quantiles. The resulting bounds tighten in high-risk configurations (near FOV limits) and relax in safer regions. We integrate these bounds into a Formation-Aware Conformal CBF-QP with a smooth margin to enforce visibility while maintaining feasibility and tracking performance. Gazebo simulations show improved formation success rates and tracking accuracy over non-adaptive (global) CP baselines that ignore formation-dependent visibility risk, while preserving finite-sample probabilistic safety guarantees. The experimental videos are available on the \href{https://nail-uh.github.io/iros2026.github.io/}{project website}\footnote{Project Website: \url{https://nail-uh.github.io/iros2026.github.io/}}.
\end{abstract}

\section{Introduction}


Distributed vision-based leader-follower systems are critical for executing complex, collaborative tasks such as cooperative exploration \cite{chiun2025marvel} and search and rescue \cite{queralta2020collaborative}. In this architecture, multi-agent coordination is achieved by designating a leader agent, while followers rely entirely on onboard sensing to maintain desired relative formations \cite{mariottini2009vision, 11127883}. A fundamental operational requirement in such vision-based schemes is ensuring that the leader consistently remains within the follower's camera field-of-view (FOV), a constraint that becomes increasingly complex to enforce during dynamic, tightly coupled maneuvers.



Ensuring operational safety within this framework is highly challenging due to the tight coupling between vehicle dynamics and camera FOV constraints, compounded by the presence of stochastic and heteroscedastic perception uncertainties. Visibility maintenance in leader-follower formations has been studied extensively \cite{panagou2014cooperative}. Recent works \cite{11127883} formulate these FOV constraints using Control Barrier Functions (CBFs) and utilize Deep Neural Networks (DNNs) for state estimation. However, nominal CBF formulations fundamentally assume exact state feedback and fail to account for the inherent uncertainties of the perception system. Simply enforcing safety constraints using an estimated state is insufficient; a moderate estimation error near the safe-set boundary can lead to a false certification of safety, where the estimated state appears safe while the true physical state violates the visibility constraints.

Existing approaches to address perception uncertainty in CBFs typically rely on robust or stochastic formulations. Robust CBFs \cite{dean2021guaranteeing,cosner2021measurement,xu2015robustness} utilize worst-case error bounds, yielding highly conservative behavior, whereas stochastic CBFs \cite{clark2019control,santoyo2021barrier} require prior knowledge of accurate error distributions, which is rarely available for complex vision models.

\begin{figure}
    \centering
    \includegraphics[width=\linewidth]{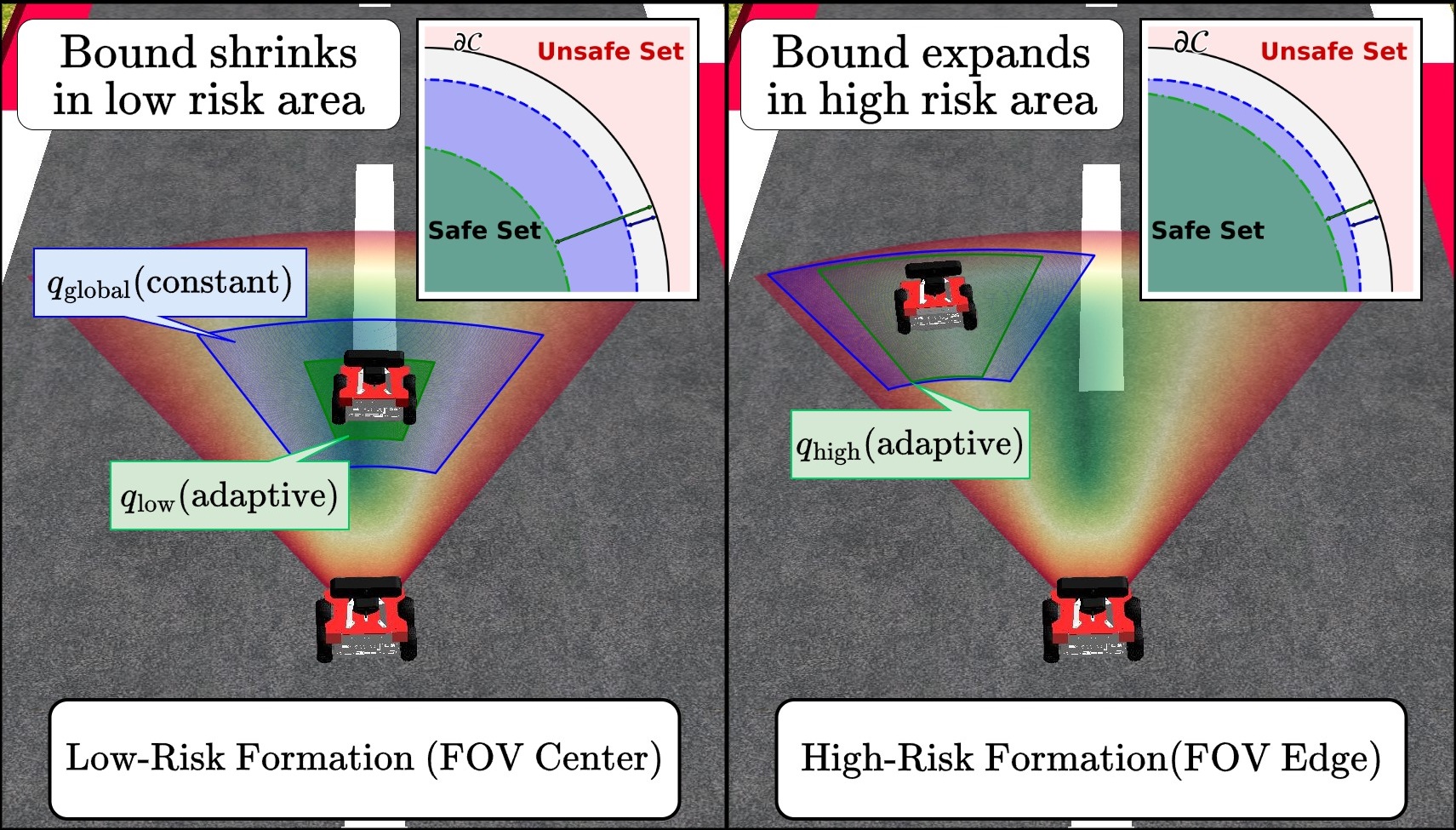}
    \caption{\textbf{Formation-Aware Adaptive Conformalized Leader–Follower System.} Under heteroscedastic, formation-dependent perception errors, the proposed risk-aware conformal bounds tighten in low-risk configurations to preserve tracking performance and expand near FOV limits to enforce probabilistic safety.}
    \label{fig:cover}
    \vspace{-5ex}
\end{figure}

Conformal Prediction (CP) has recently been integrated into safety-critical robotics as a distribution-free mechanism for translating data-driven estimation uncertainty into statistically valid safety constraints. In perception-based control, Yang \emph{et al.} \cite{yang2023safe} incorporate conformal prediction sets derived from stochastic sensor measurements into a CBF framework to obtain probabilistic safety guarantees without assuming explicit noise distributions. Zhang \emph{et al.} \cite{zhang2025conformal} embed CP-derived bounds from learned state estimators into a risk-aware CBF formulation, establishing finite-sample safety guarantees for stochastic systems. Tayal \emph{et al.} \cite{tayal2025cp} further develop CP-NCBF to statistically certify neural control barrier functions. Beyond CBF-based control, CP has been applied to safe motion planning \cite{lindemann2023safe}, safety-aware POMDP planning \cite{sheng2024safe}, adversarially robust interaction \cite{mirzaeedodangeh2025safe}, and uncertainty-aware model predictive control \cite{zhou2024safety}. These works collectively demonstrate that CP enables rigorous safety certification for learning-enabled robotic systems under uncertainties.


Despite these advances, standard CP assumes exchangeability over the entire state space, resulting in a single global error bound. While methods such as Conformalized Quantile Regression (CQR) \cite{romano2019conformalized} can address varying uncertainty, they typically require retraining the base architecture as a specific quantile regressor and do not explicitly account for spatial safety constraints. In vision-based leader-follower systems, perception errors are inherently spatial and highly heteroscedastic: uncertainty increases significantly near FOV boundaries and decreases near the image center. A global conformal bound therefore enforces worst-case conservatism everywhere, unnecessarily shrinking the feasible control set even in low-risk configurations.


To bridge the gap between rigorous safety guarantees and operational mobility, we propose a distributed, formation-aware adaptive framework that operates exclusively on local relative states. We leverage Mondrian Conformal Prediction \cite{bostrom2020mondrian}, which partitions calibration data to produce category-conditional prediction intervals. By formulating a Risk-Aware Mondrian CP, our approach computes state- and formation-conditioned uncertainty quantiles. This systematically reduces conservatism by assigning stricter statistical bounds in high-risk configurations (near FOV limits) and relaxing them in low-risk regions. We seamlessly integrate these dynamic bounds into a Formation-Aware Conformal CBF-QP utilizing a continuous margin transition, explicitly enforcing FOV limits while preserving control feasibility.


The specific contributions of this paper are threefold:
\begin{enumerate}
    \item \textbf{Adaptive Uncertainty Quantification}: We introduce a Risk-Aware Mondrian CP framework explicitly conditioned on the robot's formation state, overcoming the overly conservative, formation-agnostic bounds of standard global CP.
    \item \textbf{Conformal CBF-QP Integration}: We develop a Formation-Aware Conformal CBF-QP that integrates these heteroscedastic CP bounds using a smooth margin, ensuring strict FOV constraint satisfaction while maintaining optimization feasibility.
    \item \textbf{Empirical Validation}: We validate the proposed architecture through extensive Gazebo simulations, demonstrating superior formation success rates and tracking performance compared to global CP baselines while preserving rigorous probabilistic safety guarantees.
\end{enumerate}

\section{Preliminaries}

\noindent\textbf{Notation.} Let $\mathbb{R}$ and $\mathbb{R}_+$ denote the sets of real and nonnegative real numbers. For $x\in\mathbb{R}^n$, $\|x\|$ is the Euclidean norm. The probability of an event $\mathcal{E}$ is denoted by $\mathbb{P}(\mathcal{E})$. For a continuously differentiable $h:\mathbb{R}^n\!\to\!\mathbb{R}$ and locally Lipschitz vector field $f$, the Lie derivative is $\mathcal{L}_f h(x)\triangleq \nabla h(x)^\top f(x)$. A continuous function $\alpha:(-b,a)\to\mathbb{R}$ with $a,b\in\mathbb{R}_+$ is an \emph{extended class-$\mathcal{K}$ function} ($\alpha\in\mathcal{K}_e$) if it is strictly increasing and $\alpha(0)=0$.

\subsection{Control Barrier Functions}
Consider the control-affine system
\begin{equation}\label{eq:dynamic}
\dot{x}=f(x)+g(x)u,
\end{equation}
where $x\in\mathcal{X}\subseteq\mathbb{R}^n$, $u\in\mathcal{U}\subseteq\mathbb{R}^m$, and $f$ and $g$ are locally Lipschitz. Safety is encoded by the \emph{safe set}
\begin{equation}\label{eq:safe-set}
\mathcal{C}\triangleq \{x\in\mathbb{R}^n: h(x)\ge 0\},
\end{equation}
where $h$ is continuously differentiable.

\begin{definition}[Control Barrier Function (CBF) {\cite{ames2019control}}]\label{def:cbf}
The function $h$ is a CBF for \eqref{eq:dynamic} on $\mathcal{C}$ if there exists $\alpha\in\mathcal{K}^e$ such that for all $x\in\mathcal{C}$,
\begin{equation*}\label{eq:cbf-condition}
\sup_{u\in\mathcal{U}}\big(\mathcal{L}_f h(x)+\mathcal{L}_g h(x)u+\alpha(h(x))\big)\ge 0.
\end{equation*}
The safety set $\mathcal{C}$ is forward invariant with respect to the system defined in \eqref{eq:dynamic}.
\end{definition}

\subsection{Split Conformal Prediction}
Conformal prediction provides distribution-free finite-sample uncertainty sets under exchangeability \cite{shafer2008tutorial}. Given a dataset $\mathcal{D}=\{(y^i,x^i)\}_{i=1}^n$, where $y^{i} \in \mathcal{Y}$ denotes the observed input~(e.g., feature or measurement) and $x^{i} \in \mathcal{X}$ is the corresponding unknown target (e.g., label or state). Split conformal prediction framework randomly partitioned the dataset $\mathcal{D}$ into a training set $\mathcal{D}_{\mathrm{train}}$ and a calibration set $\mathcal{D}_{\mathrm{cal}}$ of size $n_{\mathrm{cal}}$. Let $\pi_{\theta}:\mathcal{Y}\to\mathcal{X}$ denote a predictor with parameters $\theta$ trained on $\mathcal{D}_{\mathrm{train}}$. For each data point $(y^{i}, x^{i}) \in\mathcal{D}_{\mathrm{cal}}$, let $R_i \triangleq S\!\left(x^i,\pi_{\theta}(y^i)\right) \in \mathbb{R}$ denote a nonconformity score
where $S$ is a chosen function that measures the prediction error~(e.g. $S(x,\hat{x})=\|x-\hat{x}\|$). Given a target miscoverage rate $\delta\in(0,1)$, the empirical quantile
$
\hat{q} \triangleq \mathrm{Quantile}_{1-\delta}\big(\{R_i\}_{i\in\mathcal{D}_{\mathrm{cal}}}\big)
$ of the calibration scores is computed as the $\lceil (n_{\mathrm{cal}}+1)(1-\delta)\rceil$-th smallest score in $\{R_i\}_{i\in\mathcal{D}_{\mathrm{cal}}}$. For a new observation $y^{n+1}$, the split-conformal prediction set is constructed as 
\begin{equation*}\label{eq:cp-set}
\Delta(y^{n+1}) \triangleq \{x\in\mathcal{X}: S(x,\pi_{\theta}(y^{n+1}))\le \hat{q}\}.
\end{equation*}
If the new test point $(y^{n+1}, x^{n+1})$ is exchangeable with the calibration data, this set satisfies the marginal coverage guarantee
$
\mathbb{P}\big(x^{n+1}\in \Delta(y^{n+1})\big)\ge 1-\delta.
$

\section{System Framework \& Problem Formulation}
Consider a leader--follower multi-agent system indexed by $i\in\{1,\dots,N\}$, where each follower $i$ interacts with a designated predecessor (leader) $i-1$ through the leader's transmitted control action. Let $x_i\in \mathcal{X} \subset \mathbb{R}^n$ denote the state of follower $i$, $u_i\in\mathbb{R}^m$ its control input, and let $u_{i-1}\in\mathbb{R}^m$ denote the predecessor's control input available to follower $i$ via wireless communication. The leader--follower pair subsystem $(i\!-\!1,i)$ is modeled in control-affine form as
\begin{align}
\label{eq:generic_dynamics}
\dot{x}_i = f\!\left(x_i,u_{i-1}\right) + g(x_i)u_i ,
\end{align}
where $f:\mathbb{R}^n\times\mathbb{R}^m\to\mathbb{R}^n$ captures the intrinsic follower dynamics and the coupling induced by the leader's action, and $g:\mathbb{R}^n\to\mathbb{R}^{n\times m}$ maps the follower's control input to the state dynamics. For notational simplicity, we omit any explicit dependence of $f$ and $g$ on the agent index $i$. We assume $f(\cdot,\cdot)$ is locally Lipschitz in $(x_i,u_{i-1})$ and $g(\cdot)$ is locally Lipschitz in $x_i$. This paper assumes that the follower state $x_i$ is not measured directly and is estimated online using an on-board perception pipeline, and the control input $u_i$ is synthesized based on these state estimates together with the communicated leader input $u_{i-1}$. Our objective is to develop \emph{formation-aware} uncertainty quantification for the on-board perception system and to integrate these uncertainty estimates into a control design that achieves the desired formation while enforcing \emph{perception-aware} safety constraints.

\begin{example}[Vision-based leader--follower kinematics]\label{example:system}
A concrete instance of~\eqref{eq:generic_dynamics} is the vision-based leader--follower model studied in~\cite{11127883,hu2015distributed,hu2021stochastic}. For follower $i$, define the relative state
$x_i=[L_i,\alpha_i,\phi_i]^\top\in\mathbb{R}^3$, where $L_i$ is the inter-agent distance, $\alpha_i$ is the leader bearing angle, and $\phi_i$ is the follower bearing angle. The control input is $u_i=[v_i,\omega_i]^\top\in\mathbb{R}^2$, and the communicated exogenous input is the leader control $u_{i-1}=[v_{i-1},\omega_{i-1}]^\top\in\mathbb{R}^2$. The relative kinematics take the control-affine form \(\dot{x}_i=f(x_i,u_{i-1})+g(x_i)u_i\), with
\begin{align}
\label{eq:dynamic_system}
\begin{bmatrix}\dot{L}_i\\ \dot{\alpha}_i\\ \dot{\phi}_i\end{bmatrix}
&=
\underbrace{\begin{bmatrix}
-\cos\phi_{i} & -d \sin\phi_{i}\\
-\frac{\sin\phi_{i}}{L_i} & \frac{d\cos\phi_{i}}{L_i}\\
\frac{\sin\phi_i}{L_i} & -\frac{d\cos\phi_i}{L_i}
\end{bmatrix}
\begin{bmatrix}v_i\\ \omega_i\end{bmatrix}}_{g(x_i)u_i}
+
\underbrace{\begin{bmatrix}
\cos\alpha_i & 0\\
-\frac{\sin\alpha_i}{L_i} & 1\\
\frac{\sin\alpha_i}{L_i} & 0
\end{bmatrix}
\begin{bmatrix}v_{i-1}\\ \omega_{i-1}\end{bmatrix}}_{f(x_i,u_{i-1})},
\end{align}
where $d$ is the camera offset.
\end{example}

\noindent\textbf{Perception-based Formation Tracking.} Provided desired formation setpoints, $x_{i,d}$, let $k(x_i,u_{i-1},x_{i,d})$ denote a nominal formation control law under which each leader--follower dynamics described in \eqref{eq:generic_dynamics} achieves the desired formation setpoints $x_{i,d}$ exponentially. Instead of using true state $x$ to construct the control input, the follower relies on the on-board perception system to estimate the state $x$. Let $y_i \in \mathcal{Y}$ denote the observations from the onboard sensors~(e.g., camera and IMU), and the estimate $\hat{x}_{i}$ for leader-follower system $i$ is produced by its on-board perception module, $\pi_{\theta}: \mathcal{Y} \rightarrow \mathbb{R}^{n}$, i.e., $\hat{x}_{i} = \pi_{\theta} (y_i)$. Thus, the control signal under the formation control law $k$ and perception system $\pi_{\theta}$ is  $u_{i} = k(\pi_{\theta}(y_i),u_{i-1},x_{i,d})$. Consider the tracking state $z_{i} = [L_i; \alpha_i]$ of the leader-follower dynamics defined in \eqref{eq:dynamic_system}, the formation control law $k(\cdot)$ can be designed as~\cite{hu2015distributed,11127883}
\begin{align}
\label{eq:controller}
u_i =  g^{-1}(z_{i})\left(\begin{bmatrix}
K_{L_i} & 0 \\
0 & K_{\alpha_i}
\end{bmatrix} \begin{bmatrix}
L_{i,d}-L_i \\
\alpha_{i,d}-\alpha_i
\end{bmatrix} - f(z_{i},u_{i-1})\right)
\end{align}
with controller gains $K_{L_i}, K_{\alpha_i} > 0$.

\noindent\textbf{Perception-based Safety.}
In Perception-Based  Formation Tracking, the closed loop $u_i=k(\pi_\theta(y_i),u_{i-1},x_{i,d})$ is only meaningful when the leader remains in a \emph{perception-operational domain} where the onboard module $\pi_\theta$ can reliably produce $\hat{x}_i=\pi_\theta(y_i)$. We encode this requirement via a follower-specific perception-safe set $\mathcal{C}_i\subset\mathbb{R}^n$ and require the \emph{true} state to satisfy $x_i(t)\in\mathcal{C}_i$ for all $t\ge 0$ with a pre-defined probability. As in CBF formulations, let
\begin{align*}
\mathcal{C}_i \triangleq \{x_i\in\mathbb{R}^n \mid h_{\ell}(x_i)\ge 0,\ \forall \ell\},
\end{align*}
where each $h_{\ell}:\mathbb{R}^n\to\mathbb{R}$ is continuously differentiable and locally Lipschitz.
\begin{example}[camera FOV constraints]
For the vision-based model~\eqref{eq:dynamic_system}, observability induces range and FOV limits. With distance bounds $D_{\min},D_{\max}>0$ and horizontal FOV $2\psi_{\max}$, define
\begin{align}
    \label{eq:fov-safe-set}
    \mathcal{C}_i \triangleq \left\{ (L_i, \phi_i) \;\middle|\; D_{\min} \le L_i \le D_{\max},\; |\phi_i| \le \psi_{\max} \right\}
\end{align}
equivalently captured by 
\begin{align*}
    \mathcal{C}_i \triangleq
    \begin{cases}
        \begin{aligned}
            h_1(L_i) &= L_i - D_\text{min} \geq 0 \\
            h_2(L_i) &= -L_i + D_\text{max} \geq 0 \\
            h_3(\phi_i) &= \phi_i + \psi_\text{max} \geq 0 \\
            h_4(\phi_i) &= -\phi_i + \psi_\text{max} \geq 0 \\
        \end{aligned}
    \end{cases}
\end{align*}
\end{example}

Standard CBF filters enforce forward invariance of $\mathcal{C}_i$ under full-state feedback. However, safety must hold for the true state $x_i$ while both control and safety constraints are computed from the perception output $\hat{x}_i=\pi_\theta(y_i)$. Because the estimation error $x_i-\hat{x}_i$ can be large and strongly formation-dependent, treating $\hat{x}_i$ as ground truth (as commonly assumed, e.g.,~\cite{11127883}) can lead to violated perception constraints and loss of observability. This paper addresses this gap by introducing a \emph{formation-aware adaptive conformalized perception} architecture (Fig.~\ref{fig:framework}) that calibrates and adapts uncertainty quantification to the heterogeneous risk associated with different formations, and then propagates these uncertainty sets into the safety layer to enforce perception safety. The next section details each component of the proposed architecture.

\begin{figure*}[t]
	\begin{center}
		\medskip
		\includegraphics[width=\linewidth]{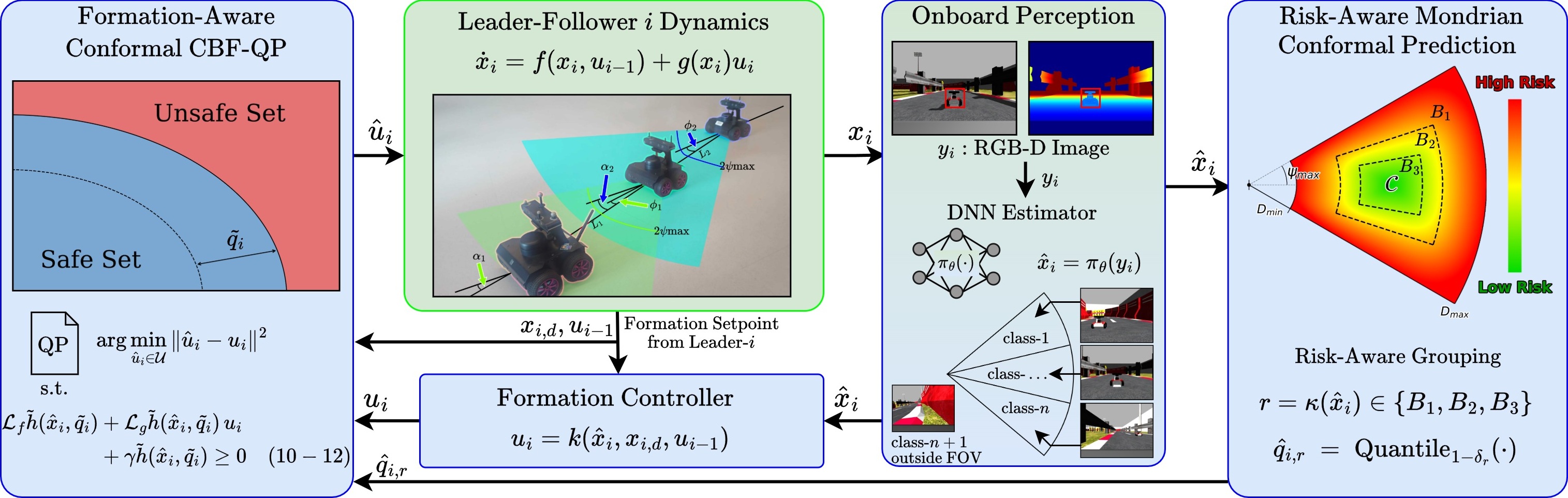}
		\caption{{\textbf{Risk-Aware Adaptive Conformal Control for Safe Leader–Follower Formation Architecture}. Onboard sensing feeds a risk-aware Mondrian conformal predictor that produces state-dependent uncertainty bounds, which parameterize the formation-aware conformal CBF-QP safety filter to enable safe tracking under perception uncertainty.}
			\label{fig:framework}}
	\end{center}
	\vspace{-4ex}
\end{figure*}
\section{Methodology}
The proposed architecture (Fig. \ref{fig:framework}) extends the leader-follower framework in \cite{11127883} by integrating perception uncertainty into the safety layer. This integration translates heteroscedastic perception errors into probabilistic safety constraints, dynamically balancing tracking performance and visibility maintenance.

\subsection{Onboard Perception}
In distributed leader-follower systems, agents lack direct acces to the true relative state $x_i\in\mathcal{X}$. Instead, each agent relies on their onboard sensors to infer the state from high-dimensional sensor observations $y_i\in\mathcal{Y}$. Let $\pi_\theta:\mathcal{Y}\to\mathcal{X}$ denote a generalized estimator parameterized by $\theta$. The onboard state estimation process is formally defined by:
\begin{align}
    \label{eq:state_estimator}
    \hat{x}_i = \pi_\theta(y_i)
\end{align}
The estimator's performance is characterized by the unmeasured estimation error $x_i-\hat{x}_i$ due to the imperfect model.

\begin{example}[Vision-based leader-follower estimation]
    consider each follower in Example~\ref{example:system} equipped with an onboard RGB-D camera, where the observation $y_i$ comprises an RGB image and its aligned depth image. The perception pipeline maps $y_i$ to the relative-state estimate $\hat{x}_i = [\hat{L}_i,\hat{\alpha}_i,\hat{\phi}_i]^\top$.The relative distance $\hat{L}_i$ is obtained by tracking a bounding box around the leader in the RGB image, mapping this region to the aligned depth image, and estimating distance as the mean depth of the pixels within the bounding box. The bearing angle $\phi_i$ is estimated with a deep neural network (DNN) $\pi_\theta$ by discretizing the camera FOV $2\psi_{\max}$ into $n$ uniform angular bins. Each bin corresponds to one class label, and the DNN predicts the most likely bin given the input image. The bearing estimate $\hat{\phi}_i$ is then taken as the center angle of the predicted bin. In this way, the continuous angle estimation problem is treated as a regression task implemented via a classification-based DNN.
    The leader bearing angle is then computed as $\hat{\alpha}_i=\theta_i-\theta_{i-1}-\hat{\phi}_i$, where $\theta_i$ denotes the follower's heading estimated via the IMU, $\theta_{i-1}$ denotes the leader heading received through wireless communication, and $\hat{\phi}_i$ is the camera-frame bearing estimated by the DNN. Additional details of the perception pipeline are provided in\cite{11127883}.
\end{example}

The complete sensing pipeline, combining depth-based range estimation and the DNN estimator $\pi_\theta$, inevitably incurs data-driven, heteroscedastic errors that vary with viewpoint, range, and image geometry. Since the resulting error distribution is unknown, enforcing safety constraints directly on the estimated state, i.e., checking $h_\ell(\hat{x}) \ge 0$, can be unsafe: even small unmodeled errors near the constraint boundary may trigger violations. This motivates explicitly quantifying perception uncertainty and incorporating it into the safety filter. To this end, we adopt \emph{Mondrian conformal regressors} (MCR) \cite{bostrom2020mondrian}, which provide distribution-free, finite-sample guarantees and naturally yield state-dependent bounds suited to heteroscedastic perception errors.

\subsection{Risk-Aware Mondrian Conformal Prediction}
\label{sec:conformal-prediction}

\begin{figure}[t]
\centering
\includegraphics[width=\linewidth]{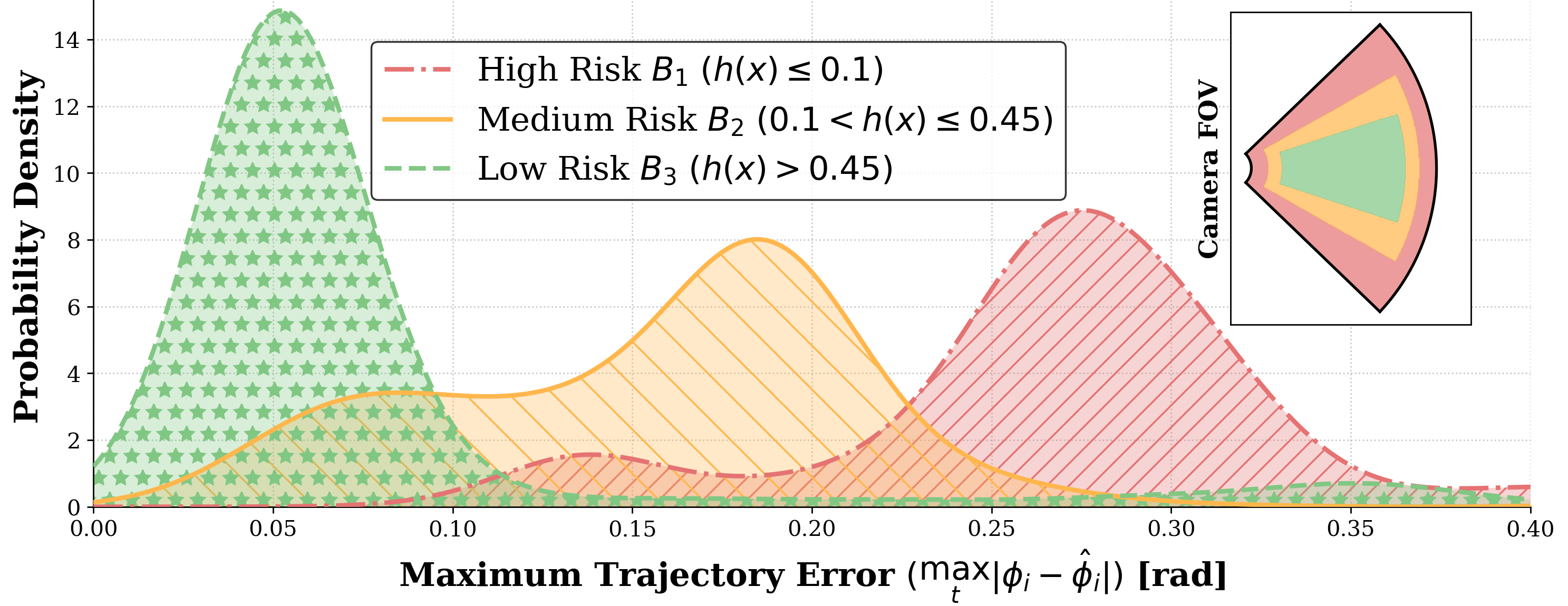}
\caption{Empirical nonconformity-score distributions across formation-dependent risk regions. The scores are strongly heteroscedastic: the high-risk near-boundary set $B_1$ exhibits substantially larger variance and heavier tails than the low-risk interior set $B_3$. This pronounced risk-dependent error structure motivates a formation-aware (risk-aware) conformal calibration that assigns region-specific quantiles, rather than a single global bound.}
\label{fig:heteroscedastic}
\vspace{-4ex}
\end{figure}

To rigorously quantify the unknown perception uncertainty, we build on the MCR framework. As shown empirically in Fig.~\ref{fig:heteroscedastic}, the perception error is strongly heteroscedastic and increases as the formation approaches the FOV boundary. In contrast, standard conformal prediction uses a single global quantile, which (i) enforces a worst-case margin in low-risk configurations and unnecessarily restricts mobility, yet (ii) can still be insufficiently conservative in high-risk, near-boundary regimes where the error distribution has heavier tails. MCR addresses this mismatch by calibrating uncertainty sets \emph{conditionally} through a partition of the state space (or risk regions), producing region-dependent bounds that better reflect formation-dependent variance.

\begin{assumption}[Trajectory Exchangeability]
\label{ass:exchangeable}
given a calibration set
$\mathcal{D}_{i,\mathrm{cal}}=\{(x_i^j(\cdot),\hat{x}_i^j(\cdot))\}_{j=1}^{n_{i,\mathrm{cal}}}$,
where for each $j$, $(x_i^j(\cdot),\hat{x}_i^j(\cdot))$ is a trajectory pair on a finite horizon
$[0,T]$, with $x_i^j:[0,T]\to\mathbb{R}^n$ the ground-truth state of follower $i$ and
$\hat{x}_i^j:[0,T]\to\mathbb{R}^n$ the corresponding perception estimate produced by $\pi_\theta$.
The calibration trajectory pairs and the test pair $(x_i(\cdot),\hat{x}_i(\cdot))$ are exchangeable in the sense that all are drawn i.i.d.\ from an unknown but fixed
distribution induced by the same controller, sensing/perception stack, environment class,
and initial-state distribution.
\end{assumption}

With Assumption~\ref{ass:exchangeable}, we exploit the empirical observation in Fig.~\ref{fig:heteroscedastic} that the perception error is strongly correlated with the formation state---in particular, it grows as the formation approaches the perception-safe boundary. To capture this heteroscedastic, state-dependent behavior, we introduce a \emph{Mondrian taxonomy} $\kappa:\mathcal{X}\to\{1,\dots,R\}$ that partitions the state space into $R$ risk groups $B_1,\dots,B_R$ using a scalar \emph{risk indicator} that measures proximity to the safe-set boundary. Concretely, let $ h(x)\triangleq \min_{\ell} h_\ell(x),$
so that smaller $h(x)$ indicates closer proximity to constraint activation. Given thresholds $\tau_1<\cdots<\tau_{R-1}$, we define
$
B_1=\{x:h(x)<\tau_1\},$
$B_r=\{x:\tau_{r-1}\le h(x)<\tau_r\}$ for $2\le r\le R-1)$, and
$B_R=\{x:h(x)\ge \tau_{R-1}\},$
and assign the online estimate $\hat{x}_i(t)$ to a risk level via $\kappa(\hat{x}_i(t))=r \iff \hat{x}_i(t)\in B_r$.

To calibrate uncertainty \emph{within} each risk group, we extract from the calibration set those trajectories that visit group $B_r$:
\begin{align*}
\mathcal{D}_{i,r} \coloneqq \Big\{ j\in\{1,\dots,n_{i,\mathrm{cal}}\}:\exists\,t\in[0,T]\ \text{s.t.}\ \kappa\big(\hat{x}_i^j(t)\big)=r \Big\}.
\end{align*}
For each $j\in\mathcal{D}_{i,r}$, define the time indices spent in $B_r$ as
$\mathcal{T}_{i,r}^j\coloneqq\{t\in[0,T]:\kappa(\hat{x}_i^j(t))=r\}$ (nonempty by construction), and compute a worst-case nonconformity score over that interval:
\begin{align}
\label{eq:non-conformity-score}
E^{j}_{i,r} = \max_{t\in\mathcal{T}_{i,r}^j}\bigl\|x_i^j(t)-\hat{x}_i^j(t)\bigr\|,\qquad \forall j\in\mathcal{D}_{i,r}.
\end{align}

Standard Mondrian CP typically uses a \emph{single} miscoverage level $\delta$ for all groups. Here, however, the \emph{safety consequence} of perception error is itself risk-dependent: near the boundary, small errors can trigger constraint violations, while interior configurations can tolerate larger errors with minimal safety impact. We therefore propose a \emph{risk-aware} calibration that assigns group-specific miscoverage levels $\delta_r$ (strict in high-risk groups, relaxed in low-risk groups) and computes the corresponding group quantiles
\begin{align}
 \label{eq:quantile}
\hat{q}_{i,r}=\mathrm{Quantile}_{1-\delta_r}\Big(\{E^{j}_{i,r}\}_{j\in\mathcal{D}_{i,r}}\Big).
\end{align}
At runtime, with active group $r=\kappa(\hat{x}_i(t))$, we form the risk-aware conformal uncertainty set
\begin{align*}
\Delta_{i}\big(\hat{x}_i(t)\big)
=\Big\{x\in\mathcal{X}:\ \|x-\hat{x}_i(t)\|\le \hat{q}_{i,r}\Big\}.
\end{align*}
Under Assumption~\ref{ass:exchangeable} and Mondrian conformalization, this construction guarantees group-conditional coverage, $\forall r\in\{1,\dots,R\}$
\begin{align*}
\mathbb{P}\!\left(x_i(t)\in \Delta_i\big(\hat{x}_i(t)\big)\ \middle|\ \kappa(\hat{x}_i(t))=r\right)\ge 1-\delta_r.
\end{align*}
\subsection{Conformal Formation-Aware CBFs}
\label{sec:safety-filter}
The nominal CBF condition certifies forward invariance of the safe set $\mathcal{C}_i$ under exact state feedback $x_i$. Under perception uncertainty, however, enforcing $h_\ell(\hat{x}_i)\ge 0$ is not sufficient: near the boundary, small estimation errors can result in $h_\ell(x_i)<0$ even when $h_\ell(\hat{x}_i)\ge 0$. To obtain a valid safety certificate, we incorporate the risk-aware conformal bounds $\{\hat{q}_{i,r}\}$ from the previous subsection into a tightened barrier condition.

\noindent\textbf{Smooth coverage-preserving conformal margin.}
At time $t$, agent $i$ determines its active risk group $r=\kappa(\hat{x}_i(t))$ and the corresponding conformal radius $\hat{q}_{i,r}$. Because $\kappa(\cdot)$ is defined by hard partitions, directly switching between $\hat{q}_{i,r}$ can introduce discontinuities (and discontinuous gradients) in the CBF constraints. Moreover, we must preserve the conformal coverage guarantees.
To address both issues, we construct a continuous margin $\tilde{q}_i(\hat{x}_i)$ that (i) varies smoothly across adjacent risk groups and (ii) \emph{over-approximates} the active group quantile, so that $\tilde{q}_i(\hat{x}_i)\ge \hat{q}_{i,\kappa(\hat{x}_i)}$ and coverage is retained.

Let $h(\hat{x}_i)$ denote the scalar risk indicator used in the taxonomy (i.e., $h(\hat{x}_i)=\min_\ell h_\ell(\hat{x}_i)$), and let $\tau_r$ be the threshold separating the higher-risk group $r$ from the adjacent lower-risk group $r{+}1$. Define a transition buffer
$\mathcal{I}_r \coloneqq [\tau_r,\tau_r+\epsilon]$ with user-chosen width $\epsilon>0$.
We then linearly interpolate between the adjacent quantiles over the transition buffer $\mathcal{I}_r$ (see Fig.~\ref{fig:quantile-comp} for the three-risk-group case):
\begin{align}
\label{eq:smooth_margin}
\tilde{q}_i(\hat{x}_i) =
\begin{cases}
\hat{q}_{i,r} + \dfrac{\hat{q}_{i,r+1}-\hat{q}_{i,r}}{\epsilon}\big(h(\hat{x}_i)-\tau_r\big),
& \text{if } h(\hat{x}_i)\in \mathcal{I}_r,\\[1.25ex]
\hat{q}_{i,\kappa(\hat{x}_i)},
& \text{otherwise}.
\end{cases}
\end{align}
Since $\hat{q}_{i,r}\ge \hat{q}_{i,r+1}$ for higher-risk $r$ (stricter coverage implies larger quantiles), \eqref{eq:smooth_margin} satisfies
$\tilde{q}_i(\hat{x}_i)\ge \hat{q}_{i,r+1}$ throughout $\mathcal{I}_r$, i.e., it is conservative for the lower-risk side of the boundary and thus preserves the conformal coverage in that region while remaining continuous.

\noindent\textbf{Barrier tightening via Lipschitz continuity.} Assume each barrier $h_\ell$ is Lipschitz with constant $L_{h_\ell}>0$. For any $x_i$ satisfying
$\|x_i-\hat{x}_i\|\le \tilde{q}_i(\hat{x}_i)$, the Lipschitz property gives
\begin{align*}
h_\ell(x_i) \;\ge\; h_\ell(\hat{x}_i) - L_{h_\ell}\|x_i-\hat{x}_i\|
\;\ge\; h_\ell(\hat{x}_i) - L_{h_\ell}\tilde{q}_i(\hat{x}_i).
\end{align*}
This motivates the \emph{conformalized} (tightened) barrier
$\tilde{h}_\ell\big(\hat{x}_i\big) \coloneqq h_\ell(\hat{x}_i) -L_{h_\ell}\tilde{q}_i(\hat{x}_i)$
so that
$h_\ell(x_i)\ge \tilde{h}_\ell(\hat{x}_i)$. 
Therefore, enforcing $\tilde{h}_\ell(\hat{x}_i)\ge 0$ implies $h_\ell(x_i)\ge 0$ whenever
$\|x_i-\hat{x}_i\|\le \tilde{q}_i(\hat{x}_i)$, which holds with probability at least $1-\delta_r$
for the active risk group by the risk-aware conformal calibration.

\noindent\textbf{Conformal Formation-Aware CBF-QP.} Using the tightened barrier $\tilde{h}_\ell(\hat{x}_i)=h_\ell(\hat{x}_i)-L_{h_\ell}\tilde{q}_i(\hat{x}_i)$, we enforce the conformal CBF condition by solving, at each time step, the following minimally invasive quadratic program:
\noindent\rule{\linewidth}{0.5pt}
\noindent \textbf{Conformal Formation-Aware CBF-QP}
\begin{flalign}
\label{eq:ccbf-qp}
    & \underset{u_i \in \mathcal{U}}{\arg\min} \quad \frac{1}{2}\left\|u_i - k\big(\hat{x}_i,u_{i-1},z_{i,d}\big)\right\|^2 &&\\
    & \text{s.t.} \quad
    \mathcal{L}_f \tilde{h}_\ell(\hat{x}_i)
    + \mathcal{L}_g \tilde{h}_\ell(\hat{x}_i)\,u_i
    + \gamma_\ell\,\tilde{h}_\ell(\hat{x}_i)\;\ge\; 0, \, \forall \ell &&
    \label{ineq:cbf}
\end{flalign}
\noindent\rule{\linewidth}{0.5pt}
where $\hat{x}_i=\pi_\theta(y_i)$ is the perception estimate and $\gamma_\ell>0$ is a user-defined CBF gain. By construction, $\tilde{h}_\ell$ incorporates a risk-adaptive conformal margin $\tilde{q}_i(\hat{x}_i)$ that is continuous across risk regions and preserves the conformal coverage guarantee, yielding a probabilistic safety certificate for the true state.
\begin{theorem}[Risk-Aware Probabilistic Forward Invariance]
\label{thm:prob_safety}
Consider the leader--follower dynamics~\eqref{eq:generic_dynamics} and the safe set
$\mathcal{C}_i \triangleq \{x_i\in\mathbb{R}^n \mid h_\ell(x_i)\ge 0,\ \forall \ell\}$.
Assume Assumption~\ref{ass:exchangeable} holds, and each $h_\ell$ is locally Lipschitz.
Let $\mathcal{V}_i \subseteq \{1,\dots,R\}$ denote the set of risk groups visited by the estimated trajectory
$\{\hat x_i(t)\}_{t\in[0,T]}$, i.e., $\mathcal{V}_i=\{r:\exists\,t\in[0,T],\ \kappa(\hat x_i(t))=r\}$ with coverage rates $\{\delta_{r}\}_{r \in \mathcal{V}_{i}}$.
If the control input $u_i(t)$ is computed for all $t\in[0,T]$ by the conformal formation-aware CBF-QP in~\eqref{eq:ccbf-qp}-\eqref{ineq:cbf}, and the initial tightened constraints satisfy $\tilde{h}_\ell(\hat{x}_i(0))\ge 0$ for all $\ell$, then
\begin{align}
\label{eq:prob_safety_conditional}
\mathbb{P}\!\left( x_i(t)\in\mathcal{C}_i,\ \forall t\in[0,T]\ \middle|\ \mathcal{V}_i \right)
\;\ge\; 1-\sum_{r\in\mathcal{V}_i}\delta_r.
\end{align}
Since $\sum_{r\in\mathcal{V}_i}\delta_r\le \sum_{r=1}^R\delta_r$, we also have the unconditional bound
$\mathbb{P}\big(x_i(t)\in\mathcal{C}_i,\ \forall t\in[0,T]\big)\ge 1-\sum_{r=1}^R\delta_r$.
\end{theorem}

\begin{proof}
For each risk group $r$, define the (trajectory-level) miscoverage event
$
\mathcal{E}_r \triangleq
\left\{\max_{\substack{t\in[0,T]\\ \kappa(\hat x_i(t))=r}}
\|x_i(t)-\hat x_i(t)\| > \hat q_{i,r}\right\}.
$
By the nonconformity score~\eqref{eq:non-conformity-score}, the risk-aware quantiles~\eqref{eq:quantile},
and Assumption~\ref{ass:exchangeable}, we have $\mathbb{P}(\mathcal{E}_r)\le \delta_r$.

On the event $\bigcap_{r\in\mathcal{V}_i}\mathcal{E}_r^c$, the estimation error satisfies
$\|x_i(t)-\hat x_i(t)\|\le \hat q_{i,\kappa(\hat x_i(t))}\le \tilde q_i(\hat x_i(t))$ for all $t\in[0,T]$.
Using local Lipschitzness of $h_\ell$ with constant $L_{h_\ell}$,
$
h_\ell(x_i(t)) \ge h_\ell(\hat x_i(t)) - L_{h_\ell}\|x_i(t)-\hat x_i(t)\|
\ge h_\ell(\hat x_i(t)) - L_{h_\ell}\tilde q_i(\hat x_i(t))
= \tilde h_\ell(\hat x_i(t)).
$
Since the CBF-QP enforces $\tilde h_\ell(\hat x_i(t))\ge 0$ for all $t$ and $\tilde h_\ell(\hat x_i(0))\ge 0$,
it follows that $h_\ell(x_i(t))\ge 0$ for all $\ell$ and all $t\in[0,T]$, i.e., $x_i(t)\in\mathcal{C}_i$.
Therefore, a safety violation implies that at least one visited risk group experiences miscoverage:
$
\left\{\exists t\in[0,T]: x_i(t)\notin\mathcal{C}_i\right\}\subseteq \bigcup_{r\in\mathcal{V}_i}\mathcal{E}_r.
$
Conditioning on $\mathcal{V}_i$ and applying Boole's inequality gives
$
\mathbb{P}\!\left(\exists t:\ x_i(t)\notin\mathcal{C}_i\ \middle|\ \mathcal{V}_i\right)
\le \mathbb{P}\!\left(\bigcup_{r\in\mathcal{V}_i}\mathcal{E}_r\ \middle|\ \mathcal{V}_i\right)
\le \sum_{r\in\mathcal{V}_i}\mathbb{P}(\mathcal{E}_r)
\le \sum_{r\in\mathcal{V}_i}\delta_r,
$
and taking complements yields~\eqref{eq:prob_safety_conditional}.
\end{proof}
\begin{remark}[Miscoverage vs.\ Safety]
\label{remark:miscoverageVsafety}
The bound in Theorem~\ref{thm:prob_safety} is conservative for two reasons. First, Boole's inequality yields a worst-case union bound. Second, the analysis effectively upper-bounds safety failure by perception miscoverage, i.e., it treats any event $\|x_i(t)-\hat{x}_i(t)\|>\tilde q_i(\hat{x}_i(t))$ as potentially safety-critical. In practice, miscoverage occurring well inside $\mathcal{C}_i$ typically does not cause a constraint violation. Safety is threatened primarily when estimation errors occur near the boundary of $\mathcal{C}_i$. Consequently, the empirical probability of maintaining safety is often substantially higher than $1-\sum_{r\in\mathcal{V}_i}\delta_r$, This is reflected in Fig.~\ref{fig:sim1}, where the observed success rate is about $95\%$, whereas the union-bound guarantee evaluates to $1-(0.01+0.1+0.45) = 44\%$, which is markedly more conservative. 
\end{remark}

For reference, Algorithm~\ref{alg:cbf_filter} summarizes the proposed framework. At each control step, the follower (i) computes the perception estimate $\hat{x}_i$ and nominal tracking input, (ii) evaluates the risk indicator $h(\hat{x}_i)$ to select the active Mondrian group, (iii) interpolates the corresponding conformal margin to ensure smoothness across group boundaries, and (iv) solves conformal formation-aware CBF-QP using this adaptive margin as a minimally invasive safety filter.

\begin{algorithm}[t]
\caption{Risk-Aware Adaptive Conformal Control for Safe Leader–Follower Formation}
\label{alg:cbf_filter}
\begin{algorithmic}[1]
\State \textbf{Input:} measurement $y_i$, nominal setpoint $x_{i,d}$, DNN $\pi_{\theta}(\cdot)$, transition buffer $\epsilon$
\State \textbf{Offline Calibration:} Mondrian boundaries $\{\tau_r\}_{r=1}^{R-1}$, Risk-Aware quantiles $\{\hat{q}_{i,r}\}_{r=1}^R$
\Loop \text{ at each control step $t$}
    \State \textbf{1. State Estimation \& Nominal Control}
    \State $\hat{x}_i \gets \pi_{\theta}(y_i)$ \eqref{eq:state_estimator}
    \State $u_{i} \gets k(\hat{x}_i,u_{i-1},x_{i,d})$ \eqref{eq:controller}

    \State \textbf{2. Risk Evaluation}
    \State Compute risk indicator: $h(\hat{x}_i) \gets \min_\ell h_\ell(\hat{x}_i)$
    \State Determine region $r \gets \kappa(\hat{x}_i)$ such that $\hat{x}_i \in B_r$
    
    \State \textbf{3. Margin Interpolation \eqref{eq:smooth_margin}}
    \If{$r < R$ \textbf{and} $h_\ell(\hat{x}_i) \in [\tau_{r-1}, \tau_{r-1} + \epsilon]$} 
        \State $\tilde{q}_i \gets \hat{q}_{i,r} + \frac{\hat{q}_{i,r+1} - \hat{q}_{i,r}}{\epsilon} (h(\hat{x}_i) - \tau_{r})$
    \Else
        \State $\tilde{q}_i \gets \hat{q}_{i,r}$
    \EndIf
    
    \State \textbf{4. Safety Filter}
    \State $\tilde{h}(\hat{x}_i, \tilde{q}_i) \gets h(\hat{x}_i) - L_h \tilde{q}_i$ \Comment{Apply adaptive margin}
    \State $\hat{u}_i \gets \arg\min_{u_i \in \mathcal{U}} \frac{1}{2} ||\hat{u}_i - u_{i}||^2$ \text{ s.t. } \eqref{ineq:cbf}
    \State Apply safe control input $\hat{u}_i$ to the follower robot
\EndLoop
\end{algorithmic}
\end{algorithm}
\section{Simulation Results}
\label{sec:simulation_results}

This section presents Gazebo simulation results validating the proposed formation-aware, perception-safe leader–follower control framework. The setup consists of two heterogeneous mobile robots—a Husarion ROSbot 2 Pro (leader) and a Clearpath Jackal (follower)—navigating an oval race-track environment (Fig. \ref{fig:environment}). We compare against: (i) the nominal controller in \cite{11127883}, which ignores perception uncertainty; and (ii) two static global conformal prediction (CP) baselines that use a single bound everywhere, calibrated using data from only the low-risk region (an optimistic bound) or only the high-risk region (a pessimistic bound). In contrast, our risk-aware adaptive conformalized method adjusts the uncertainty bound to the formation-dependent risk level, remaining tight in low-risk configurations to reduce conservatism and expanding near FOV/range limits to preserve safety. Consequently, it achieves markedly higher safety success rates in high-risk formations while maintaining strong formation-tracking performance in low-risk regimes. Simulations were executed in a Docker container on a desktop workstation with an Intel Core i9-14900KF CPU, an NVIDIA RTX 4090 GPU, and 64 GB RAM.

\noindent\textbf{Simulation Setup.}
We simulate the vision-based leader--follower system in ~\eqref{eq:dynamic_system} with camera offset $d=0.254$\,m in the oval race-track environment shown in Fig.~\ref{fig:environment}. The leader trajectory is designed to induce three formation regimes with increasing perception risk: (i) low-risk tracking directly behind the leader (near the FOV center), (ii) medium-risk tracking during sustained turning that places the leader between the FOV center and boundary, and (iii) high-risk tracking as the leader approaches the FOV boundary. The follower uses the nominal tracking controller in \eqref{eq:controller} with a CBF safety filter, with gains $K_{L_i}=3.0$, $K_{\alpha_i}=1.85$, and $\gamma_\ell=3.0, \forall\ell$, and enforces the perception-safe set~\eqref{eq:fov-safe-set} defined by $D_{\min}=0.3$\,m, $D_{\max}=3.0$\,m, and horizontal FOV $2\psi_{\max}=1.52$\,rad. Bearing perception is provided by a ConvNeXt model~\cite{liu2022convnet} trained on 200k RGB images ($3\times224\times224$), with 21 uniformly spaced angle classes (resolution $4.4^\circ$).

\noindent\textbf{Risk-Aware Conformal Calibration and Baselines.}
Nonconformity scores are calibrated offline using 450 independent runs. A Mondrian taxonomy $\kappa(\hat{x}_i)$ partitions the domain by the barrier value $h(\hat{x}_i)$ using thresholds $\tau_1=0.1$ and $\tau_2=0.45$, yielding high-risk ($h<\tau_1$), medium-risk ($\tau_1\le h<\tau_2$), and low-risk ($h\ge\tau_2$) regions, with risk levels $\delta_1=0.01$, $\delta_2=0.10$, and $\delta_3=0.45$. Online, we smooth the discrete regional quantiles via a transition buffer of width $\epsilon=0.12$, producing a continuous margin $\tilde{q}_i(h(\hat{x}_i))$ that expands near FOV/range limits and tightens in low-risk formations (Figs.~\ref{fig:Mondrian}--\ref{fig:quantile-comp}). Two static global CP baselines, an optimistic low-risk calibration ($\delta=0.45$) and a conservative high-risk calibration ($\delta=0.01$), are used for comparison.


\begin{figure}[t]
\centering
\includegraphics[width=\linewidth]{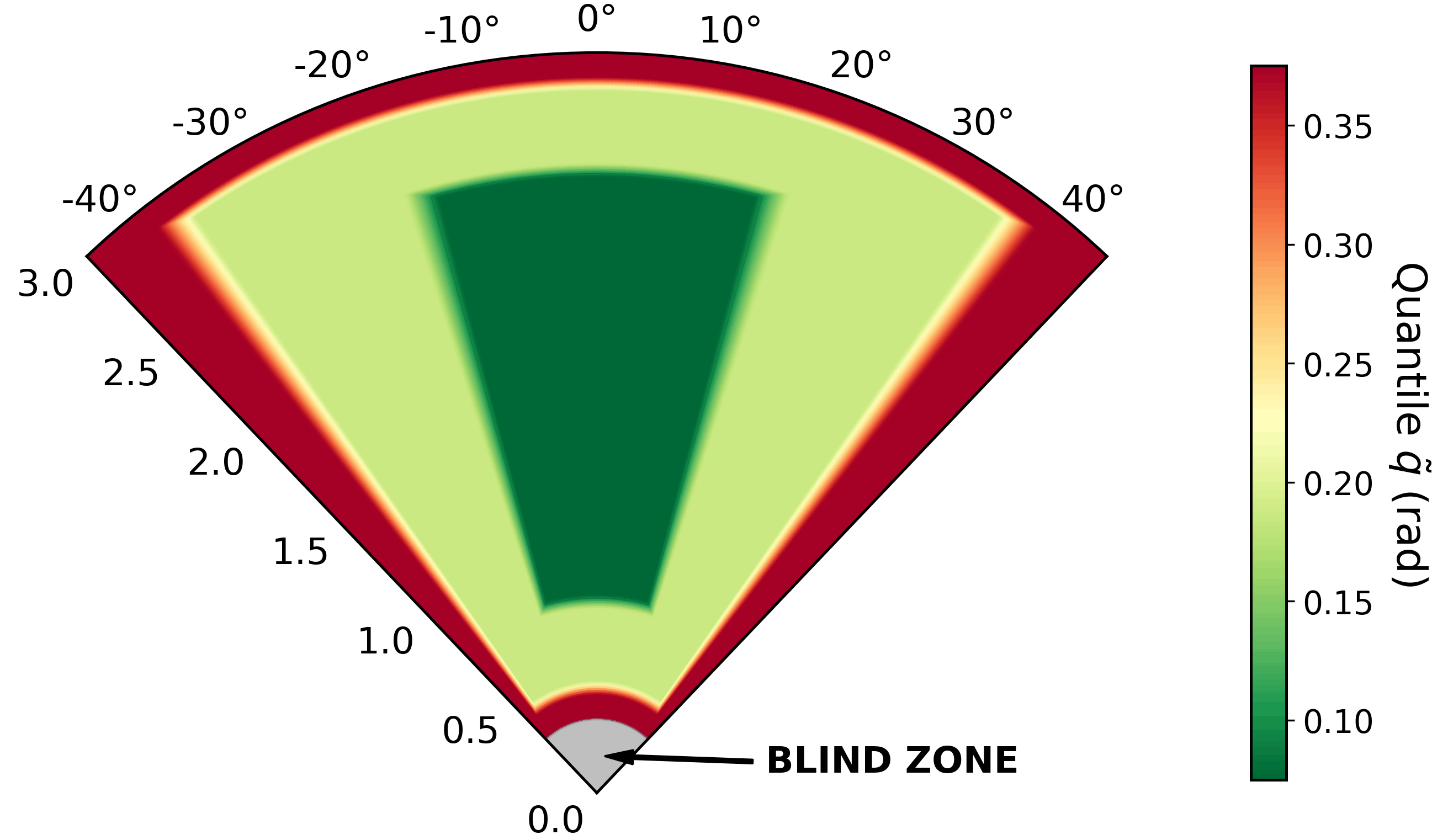}
\caption{Risk-aware Mondrian conformal quantiles projected over the camera FOV: uncertainty increases near the safety boundaries~(red) and decreases toward the interior~(green).}
\label{fig:Mondrian}
\vspace{-2.0ex}
\end{figure}

\begin{figure}[t]
\centering
\includegraphics[width=\linewidth]{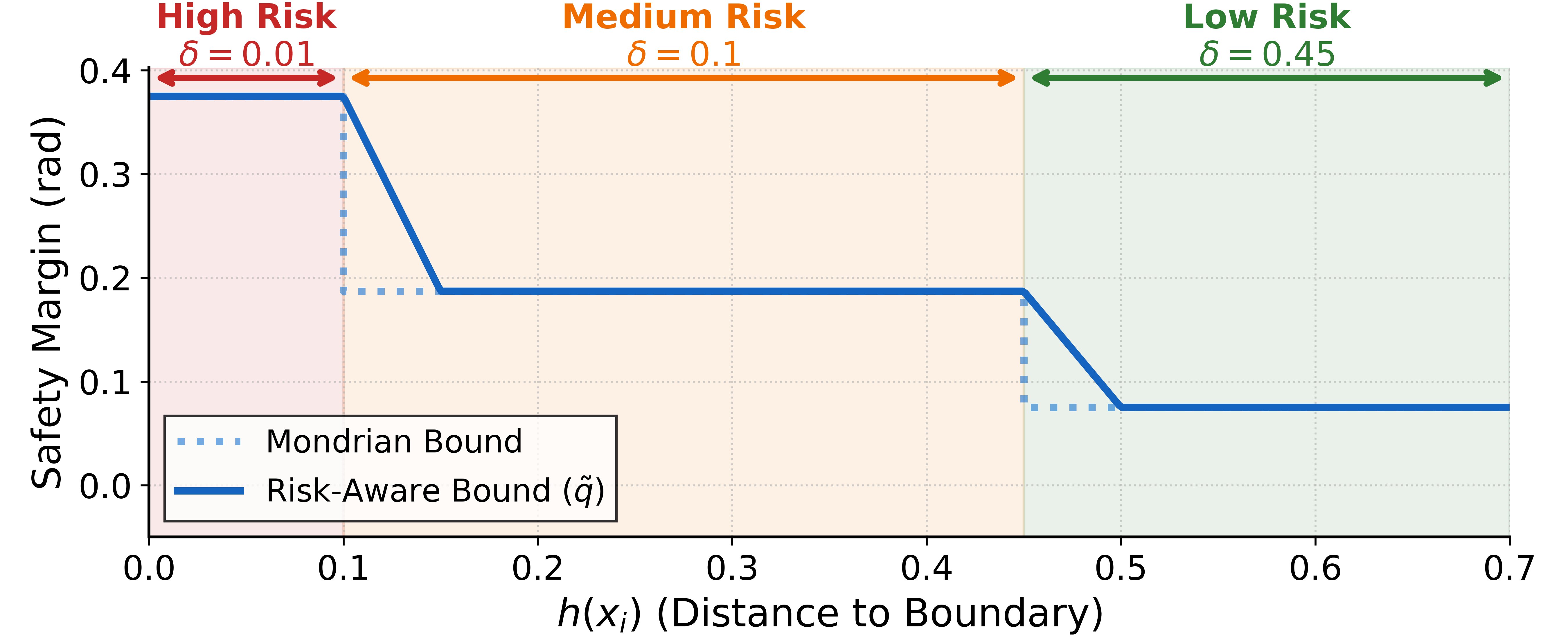}
\caption{Linear interpolation of the risk-aware conformal quantiles across Mondrian risk partitions, yielding a continuous margin $\tilde{q}_i(h)$ as in \eqref{eq:smooth_margin}.}
\label{fig:quantile-comp}
\vspace{-4ex}
\end{figure}



\begin{figure}[t]
    \centering
    \includegraphics[width=\linewidth]{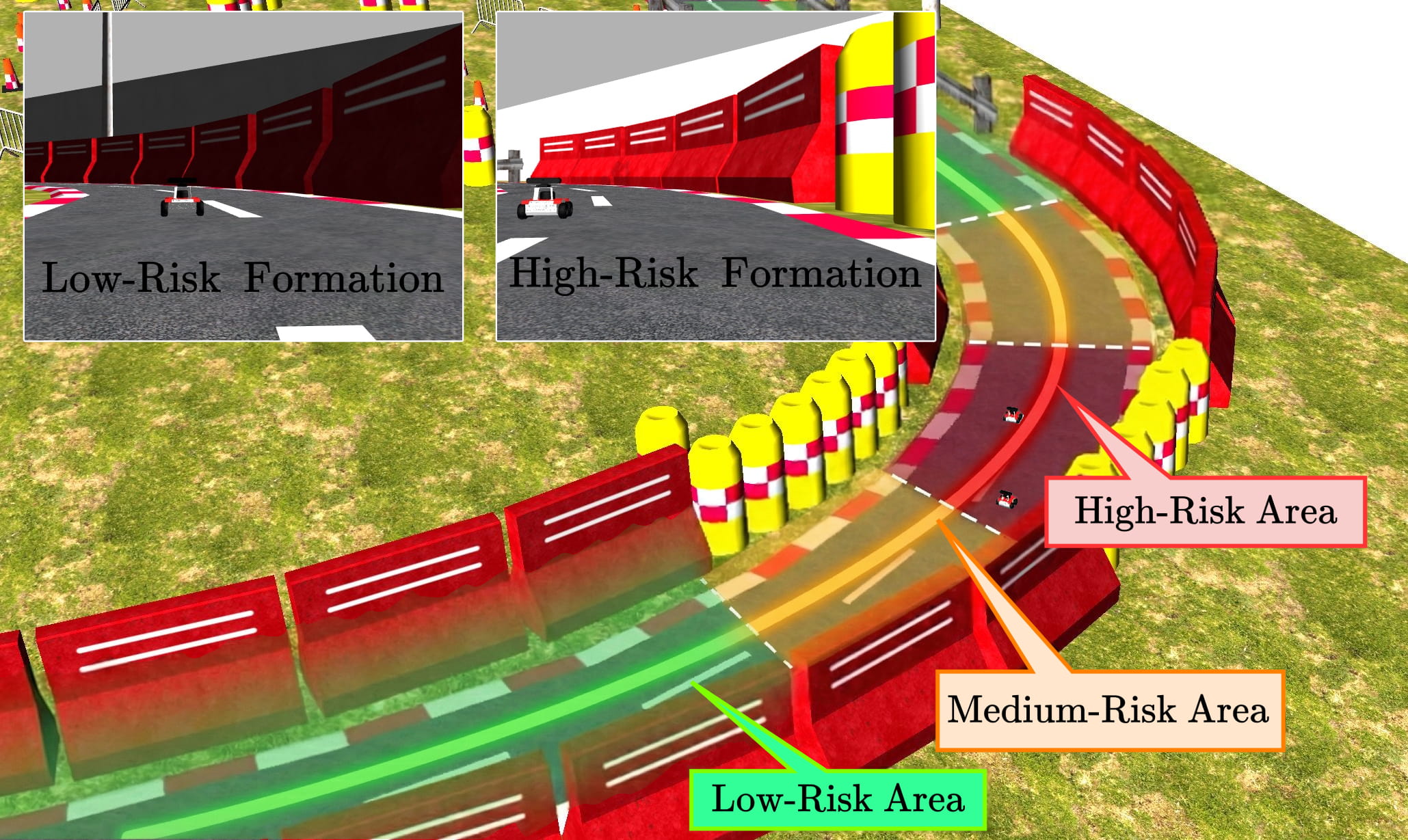}
    \caption{Gazebo race-track environment and benchmark maneuver illustrating the three formation-dependent perception risk regimes: low-risk tracking near the FOV center~(green), medium-risk during sustained turning between center and boundary~(orange), and high-risk operation near the FOV boundary~(red).}
    \label{fig:environment}
    \vspace{-4.25ex}
\end{figure}

\begin{figure*}[t]
	\centering
	\includegraphics[width=\linewidth]{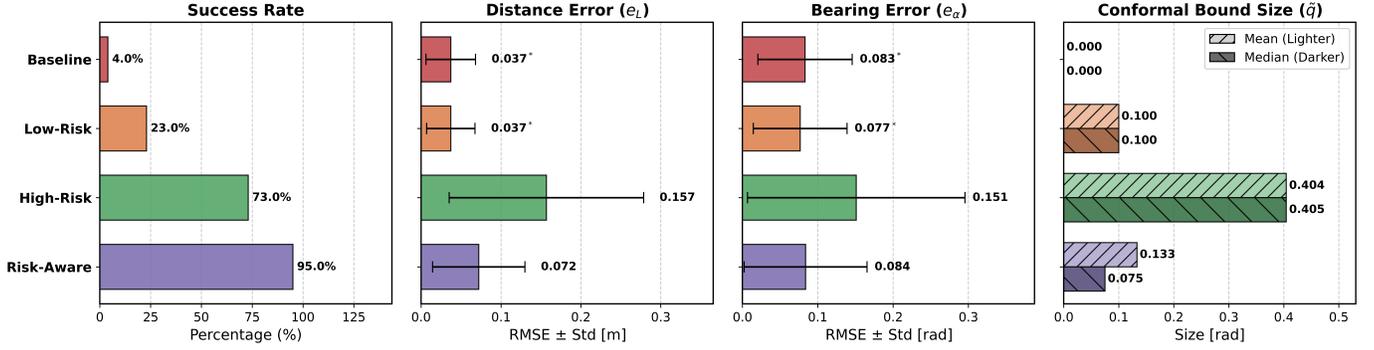}
    \caption{Performance comparison over 100 trials. The proposed Risk-Aware Method achieves the highest success rate and it also yields the best distance and bearing tracking statistics among the safe methods, while using substantially smaller conformal margins ($\tilde{q}$) than the high-risk global bound, reducing conservatism via formation-dependent adaptation. $^*$ Nominal and low-risk tracking errors are computed only on successful runs~(survivorship bias).}
	\label{fig:sim1}
    \vspace{-2ex}
\end{figure*}

\begin{figure*}[t]
	\centering
	\includegraphics[width=\linewidth]{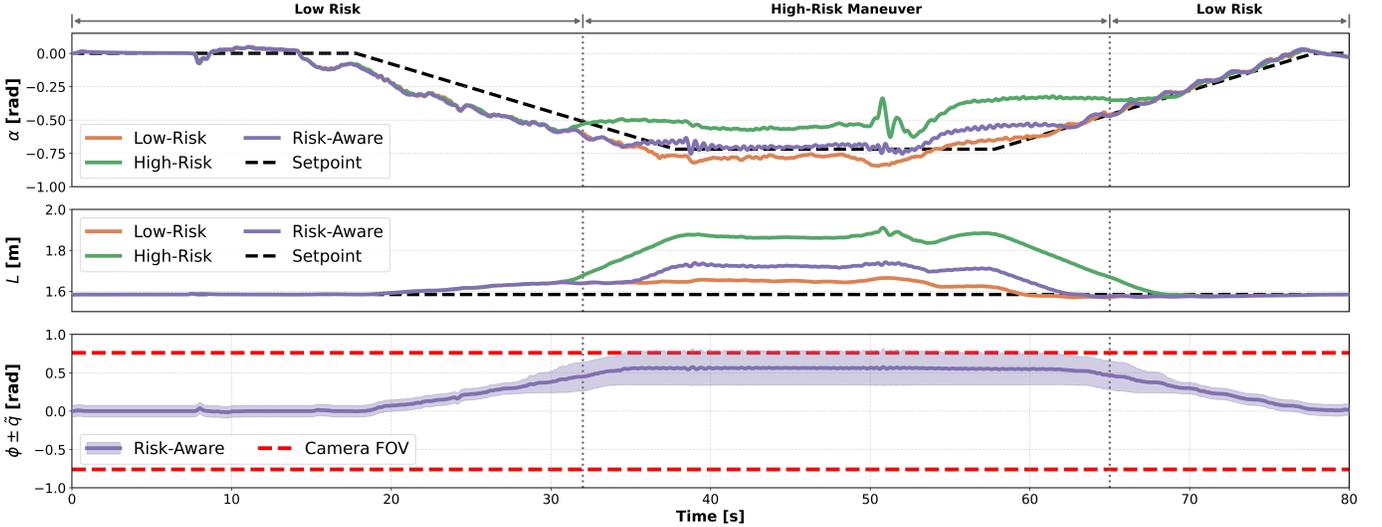}
    \caption{Average transient performance over $100$ trials. The Risk-Aware Method achieves the best transient tracking, particularly in bearing during the high-risk segment near the FOV boundary. The bottom plot shows the formation-dependent conformal margin $\tilde{q}$ shrinking in low-risk regions and expanding in high-risk regions, preserving safety while avoiding unnecessary conservatism. Averages for the nominal and global CP baselines are computed over their successful runs only.}
	\label{fig:sim1_ts}
    \vspace{-4ex}
\end{figure*}

\noindent\textbf{Validation Results.} Each method is evaluated over 100 independent trials, and  summary statistics are reported in Fig.~\ref{fig:sim1} and time-series averages in Fig.~\ref{fig:sim1_ts}.

\paragraph{Success rate}
Fig.~\ref{fig:sim1} shows that the proposed \emph{Risk-Aware Method} achieves a success rate $\mathbf{95\%}$ in the $100$ trials, significantly higher than the nominal baseline ($\mathbf{4\%}$), the low-risk global CP baseline ($\mathbf{23\%}$) and the high-risk global CP baseline ($\mathbf{73\%}$). The nominal controller fails primarily due to unmodeled perception error. The low-risk global bound underestimates uncertainty near the FOV/range limits, causing safety violations, and the high-risk global bound is overly conservative in benign configurations, frequently leading to CBF-QP infeasibility. By adapting the margin to formation risk, the risk-aware method preserves feasibility while enforcing safety.


\paragraph{Tracking Performance and Bound Size}
Fig.~\ref{fig:sim1} shows that the risk-aware method also attains the best overall tracking statistics in distance $e_L$ and bearing $e_\alpha$ (mean and variance), noting that the reported error statistics for other methods are conditioned on their \emph{successful} runs, which is consistent with their substantially lower success rates. Moreover, the risk-aware conformal uncertainty set is substantially smaller than the high-risk-only global bound on average, demonstrating reduced conservatism without sacrificing safety. Fig.~\ref{fig:sim1_ts} further confirms improved transient behavior: the risk-aware method yields the best bearing-angle tracking, especially in the high-risk segment near the FOV boundary. The bottom plot of Fig.~\ref{fig:sim1_ts} visualizes the underlying mechanism---the applied conformal margin shrinks in low-risk regions to recover control authority and expands in high-risk regions to preserve safety.

\section{Conclusion}

This paper develops a probabilistic safety framework for vision-based leader–follower formation control under heteroscedastic, formation-dependent perception errors. Since a single global conformal bound is either overly conservative in low-risk formations or insufficient near FOV/range limits, we propose a risk-aware Mondrian conformal predictor that calibrates region-specific, smoothly varying uncertainty margins based on visibility risk. These adaptive bounds are integrated into a conformal formation-aware CBF-QP safety filter, yielding a probabilistic forward-invariance guarantee while retaining control authority. Simulations demonstrate markedly higher safety success rates and improved tracking compared to nominal and static global CP baselines.
\balance
\bibliographystyle{IEEEtran}
\bibliography{IROS}

\end{document}